\documentclass{article}

\usepackage{amssymb, amsmath, amsthm}
\newtheorem{theorem}{Theorem}
\newtheorem{lemma}[theorem]{Lemma}
\newcommand{\norm}[1]{\left\lVert#1\right\rVert}

\usepackage{microtype}
\usepackage{graphicx}
\usepackage{subfigure}
\usepackage{booktabs} 

\usepackage{hyperref}

\usepackage{xcolor}

\usepackage[accepted]{icml2020}

\icmltitlerunning{Curse of Dimensionality on Randomized Smoothing for Certifiable Robustness}

\begin{document}

\twocolumn[
\icmltitle{Curse of Dimensionality on Randomized Smoothing for Certifiable Robustness}



\icmlsetsymbol{equal}{*}

\begin{icmlauthorlist}
\icmlauthor{Aounon Kumar}{umd}
\icmlauthor{Alexander Levine}{umd}
\icmlauthor{Tom Goldstein}{umd}
\icmlauthor{Soheil Feizi}{umd}
\end{icmlauthorlist}

\icmlaffiliation{umd}{University of Maryland, College Park, Maryland, USA}

\icmlcorrespondingauthor{Aounon Kumar}{aounon@umd.edu}
\icmlcorrespondingauthor{Soheil Feizi}{sfeizi@cs.umd.edu}

\icmlkeywords{Trustworthy Machine Learning, Certifiable Robustness, Randomized Smoothing}

\vskip 0.3in
]



\printAffiliationsAndNotice{}  

\begin{abstract}

Randomized smoothing, using just a simple isotropic Gaussian distribution, has been shown to produce good robustness guarantees against $\ell_2$-norm bounded adversaries. In this work, we show that extending the smoothing technique to defend against other attack models can be challenging, especially in the high-dimensional regime.  In particular, for a vast class of i.i.d.~smoothing distributions, we prove that the largest $\ell_p$-radius that can be certified decreases as $O(1/d^{\frac{1}{2} - \frac{1}{p}})$ with dimension $d$ for $p > 2$. Notably, for $p \geq 2$, this dependence on $d$ is no better than that of the $\ell_p$-radius that can be certified using isotropic Gaussian smoothing, essentially putting a matching lower bound on the robustness radius.
When restricted to {\it generalized} Gaussian smoothing, these two bounds can be shown to be within a constant factor of each other in an asymptotic sense, establishing that Gaussian smoothing provides the best possible results, up to a constant factor, when $p \geq 2$. We present experimental results on CIFAR to validate our theory.
For other smoothing distributions, such as, a uniform distribution within an $\ell_1$ or an $\ell_\infty$-norm ball, we show upper bounds of the form $O(1 / d)$ and $O(1 / d^{1 - \frac{1}{p}})$ respectively, which have an even worse dependence on $d$. 

\end{abstract}

\section{Introduction}
Deep neural networks, especially in image classification tasks, have been shown to be vulnerable to adversarial perturbations of the input that are unnoticeable to a human observer but can alter the prediction of the model \cite{Szegedy2014}. These examples are generated by optimizing a loss function for a trained network over the input features within a small neighborhood of an example input. Gradient based methods such as FGSM \cite{GoodfellowSS14} and projected gradient descent \cite{MadryMSTV18} have been shown to be very effective for this purpose. In the last couple of years, several heuristic methods have been proposed to detect and/or defend against attacks from specific types of adversaries \cite{BuckmanRRG18, GuoRCM18, DhillonALBKKA18, LiL17, GrosseMP0M17, GongWK17}.
Such defenses, however, have been shown to break down against more powerful attacks \cite{Carlini017, athalye18a, UesatoOKO18, LaidlawF19}. For certain types of problems, adversarial examples might even be unavoidable \cite{ShafahiHSFG19}.

This necessitates developing classifiers with robustness guarantees. Several convex relaxation-based techniques have been proposed to design {\it certifiably robust} classifiers \cite{WongK18, Raghunathan2018, Singla2019, Chiang20, singla2020curvaturebased} whose predictions are guaranteed to remain constant within a certified neighborhood around the input point, thereby eliminating the presence of any adversarial example in that region. However, the ever-increasing complexity of deep neural networks has made it difficult to scale these methods meaningfully to high-dimensional datasets like ImageNet.

To deal with the scalability issue in certifiable robustness, a line of work has been introduced based on {\it randomized robustness} \cite{LecuyerAG0J19, LiCWC19, cohen19, SalmanLRZZBY19, Levine2020aistats, Levine2020patch, LevineF20, LeeYCJ19, teng2020ell, zhang2020}
wherein an arbitrary base classifier is made more robust by averaging its prediction over random perturbations of the input point within its neighborhood. \citeauthor{cohen19}~(\citeyear{cohen19}) proved the first tight robustness guarantee for Gaussian smoothing for an $\ell_2$-norm bounded adversary.

In this work, however, we show that extending the smoothing technique to defend against higher-norm attacks, especially in the high-dimensional regime, can be challenging. In particular, for a general class of i.i.d.~smoothing distributions, we show that, for $p>2$, the largest $\ell_p$-radius that can be certified (denoted by $r^*_p$ ) decreases with the number of dimensions $d$ as $O(1/d^{\frac{1}{2} - \frac{1}{p}})$. Note that the special case of $p=2$ does not suffer from such dependency on $d$. This makes smoothing-based robustness bounds weak against $\ell_p$ adversarial attacks for large $p$, especially, for $\ell_\infty$ because as $p \rightarrow \infty$ the dependence on $d$ becomes $O(1/\sqrt{d})$. Moreover, we show that the dependence of the robustness certificate on $d$ using a general i.i.d.\! smoothing distribution is similar to that of the standard Gaussian smoothing, even for $p>2$. This implies that Gaussian smoothing essentially provides the best possible robustness certificate result in terms of the dependence on $d$ even for $p>2$.  

To be more precise, suppose we smooth a classifier by randomly sampling points surrounding an image $x,$ and observing the labels assigned to these points.  Let $p_1(x)$ and $p_2(x)$ be the probabilities of the first and second most probable labels under the smoothing distribution. We prove the following bounds on the robustness certificate:
\begin{enumerate}
    \item When points are sampled by adding i.i.d.\! noise to each dimension in $x$ with $\sigma^2$ variance and continuous support, we prove the certified $\ell_p$ radius bound 
    \[ r^*_p \leq \frac{\sigma}{2 \sqrt{2} d^{\frac{1}{2} - \frac{1}{p}}} \left( \frac{1}{\sqrt{1-p_1(x)}} + \frac{1}{\sqrt{p_2(x)}} \right), \]
    whenever $p_1(x) \geq 1/2$. See Theorem~\ref{thm:gen_iid}.
    
    \item When smoothing with a generalized Gaussian distribution with variance $\sigma^2$ (which includes Laplacian, Gaussian, and uniform distributions), we prove that
    \[r^*_p \leq \frac{2 \sigma}{d^{\frac{1}{2} - \frac{1}{p}}} \left(\sqrt{\log \frac{1}{1-p_1(x)} } + \sqrt{\log \frac{1}{p_2(x)}} \right), \]
    when $e^{-d/4} < p_2(x) \leq p_1(x) < 1 - e^{-d/4}$. When $d$ is large, these bounds do not impact the range of values that $p_1(x)$ and $p_2(x)$ can take in a significant way. See Theorem~\ref{thm:gen_gauss}. 
    
    \item We also study smoothing techniques where the distribution is uniform over a region around the input point. When smoothed over an $\ell_\infty$ ball of radius $b$, i.e. uniform i.i.d between $-b$ and $b$ in each dimension, we show that
    \[r^*_p < \frac{2b}{d^{1 - \frac{1}{p}}} = 2 \sqrt{3} \sigma / d^{1 - \frac{1}{p}},\]
    where $\sigma^2 = b^2/3$ is the variance in each dimension.
    See Theorem~\ref{thm:l_inf_uniform}. Note that this bound is independent of $p_1(x)$ and $p_2(x)$.
    
    \item For smoothing uniformly over an $\ell_1$ ball of the same radius $b$, we achieve an even stronger bound: 
    \[ r^*_p < \frac{2b}{d}\] 
    See Theorem~\ref{thm:l1_uniform} for details. Along with being independent of $p_1(x)$ and $p_2(x)$, it is also independent of $p$. Thus, it holds for any $p$-norm bounded adversary. Note that, unlike the other smoothing distributions we have considered, the uniform $\ell_1$ smoothing is not i.i.d.\! in every dimension.
\end{enumerate}
These bounds hold for any $p > 0$, but are too weak to offer meaningful insights when $p < 2$ in the first two cases and for $p < 1$ in the third one. Moreover, it is straightforward to show that, for $p \geq 2$, the following $\ell_p$-radius can be certified using \citeauthor{cohen19}'s~(\citeyear{cohen19}) Gaussian smoothing:
\begin{equation}
\label{eq:gaussian_bound}
    r_p  = \frac{\sigma}{2 d^{\frac{1}{2} - \frac{1}{p}}}\left(\Phi^{-1}\left(p_1(x)\right) - \Phi^{-1}\left(p_2(x)\right)\right),
\end{equation}
which has the same dependence on $d$ as the upper bound obtained using i.i.d.\! smoothing. This radius is asymptotically only a constant factor away from the upper bound for the generalized Gaussian distribution, showing that this family of distributions fails to outperform standard Gaussian smoothing in high dimensions. To the best of our knowledge, these bounds form the first results on the limitations of randomized smoothing in the high dimensional regime that cover an extensive range of natural and commonly used smoothing distributions.\footnote{We have later come to know about a concurrent work which also illustrates the difficulty of extending randomized smoothing to defend against $\ell_\infty$
-attacks for high-dimensional data \cite{blum2020random}.}
We provide empirical evidence to support our claims on the CIFAR-10 dataset.

\section{Preliminaries and Notation}
Let $h$ be a classifier that maps inputs from $\mathbb{R}^d$ to classes in $\mathcal{C}$. Let $\mathcal{P}$ be a (smoothing) probability distribution in $\mathbb{R}^d$. We define a \emph{smoothed} classifier $\bar{h}$ as below:
\[\bar{h}(x) \triangleq \underset{c \in \mathcal{C}}{\arg\max} \underset{\Delta \sim \mathcal{P}}{\mathbb{P}}(h(x + \Delta) = c). \]
We refer to the process of smoothing using distribution $\mathcal{P}$ as $\mathcal{P}$-smoothing.
Let $p_c(x)$ be the output probability of the base classifier for the class $c$. That is, \[p_c(x) := \underset{\Delta \sim \mathcal{P}}{\mathbb{P}}(h(x + \Delta) = c). \] Without loss of generality, we assume that $p_1(x)$ and $p_2(x)$ are the probabilities of the first and second most likely classes, respectively. 

For $p>0$, we say a smoothing distribution $\mathcal{P}$ achieves a \emph{certified $\ell_p$-norm radius} of $r_p$ if, for a base classifier $h$ and an input $x$,
\[ \bar{h}(x + \delta) = \bar{h}(x), \quad \forall \delta \in \mathbb{R}^d, \norm{\delta}_p \leq r_p.\]
For instance, as derived in \cite{cohen19}, the Gaussian smoothing distribution $\mathcal{N}(0, \sigma^2 I)$ achieves a certified 2-norm radius of $\frac{\sigma}{2}(\Phi^{-1}(p_1(x)) - \Phi^{-1}(p_2(x)))$ where $\Phi^{-1}$ is the inverse of the standard Gaussian CDF.

For $p_1, p_2 \in (0, 1)$, such that, $p_1 \geq p_2$, let $r^*_p$ denote the largest $r_p$ that can be certified using $\mathcal{P}$-smoothing for all classifiers satisfying $p_1(x) = p_1$ and $p_2(x) = p_2$. If we can show a  classifier $h$ in this class and two points $x, x' \in \mathbb{R}^d$, such that, $\bar{h}(x) \neq \bar{h}(x')$, then $r^*_p \leq \norm{x'-x}_p$. We use this fact to show upper bounds on the largest $p$-norm radius that can be certified using a given class of distributions.

\section{General i.i.d.\! Smoothing}
We set the $\mathcal{P}$ to be a smoothing distribution $\mathcal{I}$ where each coordinate of $\Delta$ is sampled independently and identically from a symmetric distribution with zero mean, $\sigma^2$ variance with a continuous support. We prove the following theorem:
\begin{theorem}
\label{thm:gen_iid}
For distribution $\mathcal{I}$ and for $p_1, p_2 \in (0, 1)$, such that, $p_1 \geq 1/2$ and $p_1 + p_2 \leq 1$, the largest $\ell_p$-radius $r^*_p$ that can be certified for all classifiers satisfying $p_1(x) = p_1$ and $p_2(x) = p_2$ under $\mathcal{I}$-smoothing at input point $x$ is bounded as:
\begin{equation}
\label{eq:iid_bnd}
    r^*_p \leq \frac{\sigma}{2 \sqrt{2} d^{\frac{1}{2} - \frac{1}{p}}} \left( \frac{1}{\sqrt{1-p_1(x)}} + \frac{1}{\sqrt{p_2(x)}} \right).
\end{equation}
\end{theorem}

\begin{proof}
Let $Z_i$ be the random variable modelling the $i^{th}$ coordinate of $\Delta$. Define a random variable $S = \sum_{i=1}^d Z_i$. It is straightforward to show that this random variable is distributed symmetrically with zero mean, $d \sigma^2$ variance and a continuous support. The key intuition behind this proof is that the random variable $S$, which is the sum of $d$ identical and independent random variables, will tend towards a Gaussian distribution for large values of $d$, making the distribution $\mathcal{I}$ suffer from some of the same limitations as the Gaussian distribution. 

To simplify our analysis, we move our frame of reference so that $x$ is at the origin. Therefore, $r^*_p \leq \norm{x'}_p$. Consider a classifier $g$ that maps points in $\{ w \in \mathbb{R}^d \mid \sum_{i=1}^d w_i \leq s_1 \}$ to class one and those in $\{ w \in \mathbb{R}^d \mid \sum_{i=1}^d w_i \geq s_2 \}$ to class two. We pick $s_1, s_2 \in \mathbb{R}^+$ such that, $\mathbb{P}(S \leq s_1) = p_1(x)$ (this requires $p_1(x) \geq 1/2$) and $\mathbb{P}(S \geq s_2) = p_2(x)$.
Let $x'$ be the point with every coordinate equal to $\epsilon$ and so, $\sum_{i=1}^d x'_i = \epsilon d$. Since $S$ is symmetric and has a continuous support, $\bar{g}(x') = \bar{g}(x)$ only if $\sum_{i=1}^d x'_i \leq \frac{s_1 + s_2}{2}$, which implies $\epsilon \leq \frac{s_1 + s_2}{2d}$. Therefore,
\begin{equation}
\label{robust_cert_bound}
  r^*_p \leq \norm{x'}_p = \epsilon d^{1/p} \leq \frac{s_1 + s_2}{2d^{1 - \frac{1}{p}}}.
\end{equation}
Figure~\ref{fig:pdf_shift} illustrates how the probabilities of the top two classes change as we move from $x$ to $x'$.

\begin{figure}[t]
\centering
\includegraphics[scale=0.2]{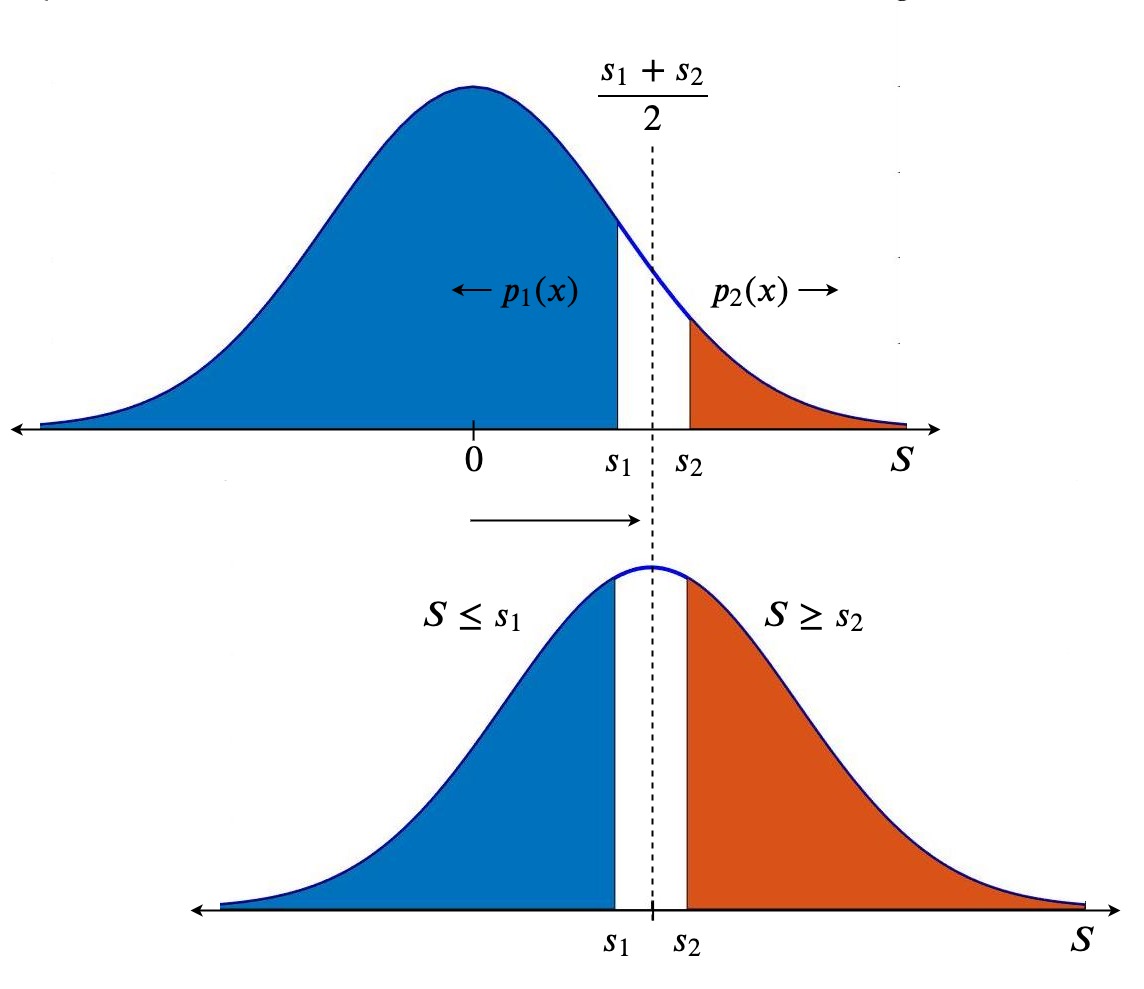}\vspace{-3mm}
\caption{As the distribution of $S$ moves from the origin to $\frac{s_1 + s_2}{2}$ the probability for class one decreases and that of class two increases. They become equal at $\frac{s_1 + s_2}{2}$ beyond which class two becomes more likely.}
\label{fig:pdf_shift}
\end{figure}

Applying Chebyshev's inequality on $S$, we have:
\[ P( S \geq s) = \frac{P(|S| \geq s)}{2} \leq \frac{d \sigma^2}{2s^2} \]
The value of $s$ for which $\frac{d \sigma^2}{2s^2} = p_2(x)$ must be an upper-bound on $s_2$.
\[s_2 \leq \frac{\sqrt{d} \sigma}{\sqrt{2p_2(x)}} \]
Similarly, since $\mathbb{P}(S \geq s_1) = 1-p_1(x)$,
\[s_1 \leq \frac{\sqrt{d} \sigma}{\sqrt{2(1-p_1(x))}} \]
Substituting the above bounds for $s_1$ and $s_2$ in~(\ref{robust_cert_bound}), proves Theorem~(\ref{thm:gen_iid}):
\[ r^*_p \leq \frac{\sigma}{2 \sqrt{2} d^{\frac{1}{2} - \frac{1}{p}}} \left( \frac{1}{\sqrt{1-p_1(x)}} + \frac{1}{\sqrt{p_2(x)}} \right). \]
\end{proof}

\section{Generalized Gaussian Smoothing}
We now restrict ourselves to the class of generalized Gaussian distributions that subsumes some commonly used and natural smoothing distributions such as Gaussian, Laplacian and uniform distributions. Using a similar approach as in the previous section, we obtain tighter upper bounds on $r^*_p$ by restricting the smoothing distribution to generalized Gaussian. In this class of distributions, each coordinate is sampled independently from the following distribution:
\[p(z) = \frac{1}{C} e^{- \left( |z|/b \right)^q} \]
where $z \in \mathbb{R}$, $b > 0$ is the \emph{scale parameter}, $q > 0$ is the \emph{shape parameter} and $C$ is the normalizing constant
\begin{align}\label{eq:C}
    C &= \int_{-\infty}^{\infty} e^{- \left( |z|/b \right)^q} dz\\
      &= 2 \int_{0}^{\infty} e^{- z^q/b^q } dz = \frac{2b \Gamma(1/q)}{q}\nonumber,
\end{align}
where $\Gamma(.)$ is the gamma function. The mean of this distribution is at zero and the variance $\sigma^2$ can be calculated as
\begin{align*}
    \sigma^2 &= \frac{1}{C} \int_{-\infty}^{\infty} z^2 e^{- \left( |z|/b \right)^q} dz\\
    &= \frac{2}{C} \int_{0}^{\infty} z^2 e^{- z^q/b^q } dz = \frac{2b^3 \Gamma(3/q)}{Cq}.
\end{align*}
Substituting $C$ from \eqref{eq:C} leads to
\[\sigma^2 = \frac{b^2 \Gamma(3/q)}{\Gamma(1/q)}.\]
Note that the class of generalised Gaussian distributions is a subset of the class of i.i.d.\! smoothing distributions considered in the previous section. The joint probability distribution over all the $d$ dimensions can be expressed as:
\[p(z_1, z_2, \ldots, z_d) = \frac{1}{C^d} e^{- \sum_{i=1}^{d} \left( |z_i|/b \right)^q},\]
which for $q = 1, 2$ represents Laplace and Gaussian distributions, respectively. As $q \rightarrow \infty$, this distribution approximates the uniform distribution over $[-b, b]^d$. For a finite $q$, the level sets of the above p.d.f.\! define sets with constant $\ell_q$-norm. Let $\mathcal{G}$ be a generalised Gaussian distribution with $q \geq 1$. The following theorem holds:
\begin{theorem}
\label{thm:gen_gauss}
For distribution $\mathcal{G}$ and for $e^{-d/4}< p_2 \le p_1 < 1 - e^{-d/4}$ and $p_1 + p_2 \leq 1$, the largest $\ell_p$-radius $r^*_p$ that can be certified for all classifiers satisfying $p_1(x) = p_1$ and $p_2(x) = p_2$ under $\mathcal{G}$-smoothing at input point $x$, is bounded as:
\begin{equation}
\label{eq:gen_gauss_bnd}
    r^*_p \leq \frac{2 \sigma}{d^{\frac{1}{2} - \frac{1}{p}}} \left(\sqrt{\log (1/(1-p_1(x)))} + \sqrt{\log (1/p_2(x))} \right)
\end{equation}
\end{theorem}

We provide a brief proof sketch for this theorem here.
As before, define random variables $Z_i$ and $S$, and assume $x$ to be at the origin. Since the above distribution satisfies all the assumptions made in the previous section, we can directly conclude that the bound in~(\ref{robust_cert_bound}) holds:
\[r^*_p \leq \frac{s_1 + s_2}{2d^{1 - \frac{1}{p}}}\]
From here, we strengthen our analysis by replacing Chebyshev's inequality with Chernoff bound.
\[P(S \geq s) \leq \frac{E[e^{tS}]}{e^{ts}}\]
for any $t>0$. Since $S$ is a sum of independent random variables $Z_1, Z_2, \ldots, Z_d$ sampled from identical distributions,
\[P(S \geq s) \leq e^{-ts} \prod_{i=1}^{d} E[e^{tZ_i}] \leq e^{-ts} E[e^{tZ}]^d\]
where $Z$ is sampled from $p(z)$.

\begin{lemma}
\label{lem:exp_bnd}
    For some constant $c < 1.85$,
    \[ E[e^{tZ}] \leq \sum_{m=0}^{\infty} (c^2 t^2 \sigma^2)^m \]
\end{lemma}
Proof is presented in the appendix.

Setting $t = \frac{1}{\tau \sigma \sqrt{d}}$ for some $\tau > 0$ satisfying $\frac{c^2}{\tau^2 d} < 1$, we have:
\begin{align*}
    P(S \geq s) &\leq e^{-s/\tau \sigma \sqrt{d}} \left( \sum_{m=0}^{\infty} (c^2/\tau^2 d)^m \right)^d\\
    &= \frac{e^{-s/\tau \sigma \sqrt{d}}}{(1-\frac{c^2}{\tau^2 d})^d} \leq e^{-s/\tau \sigma \sqrt{d}} e^{4/\tau^2}
\end{align*}
for $\tau^2 d \geq 16$. The value of $s$ for which this expression is equal to $p_2(x)$ gives us the following upper-bound on $s_2$:
\[s_2 \leq \sigma \sqrt{d} (\tau \log(1/p_2(x)) + 4/\tau)\]
which for $\tau = 2/\sqrt{\log(1/p_2(x))}$ gives:
\[s_2 \leq 4 \sigma \sqrt{d \log(1/p_2(x))}\]
and similarly, repeating the above analysis and setting $\tau = 2/\sqrt{\log(1/(1-p_1(x)))}$, we get:
\[s_1 \leq 4 \sigma \sqrt{d \log(1/(1 - p_1(x)))}\]
Both the above values for $\tau$ satisfy $\tau^2 d \geq 16$ due to the restrictions on $p_1$ and $p_2$. Substituting the above bounds for $s_1$ and $s_2$ in inequality~(\ref{robust_cert_bound}), proves Theorem~(\ref{thm:gen_gauss}):
\[r^*_p \leq \frac{2 \sigma}{d^{\frac{1}{2} - \frac{1}{p}}} \left(\sqrt{\log (1/(1-p_1(x)))} + \sqrt{\log (1/p_2(x))} \right) \]

When $p_1(x)$ is close to one and $p_2(x)$ is close to zero, this bound is within a constant factor of the Gaussian certificate in equation~(\ref{eq:gaussian_bound}) because $\Phi^{-1}(p)$ can be lower bounded by $\alpha \sqrt{\log(1/(1-p)) + \beta}$ for some constants $\alpha$ and $\beta$.
Figure~(\ref{fig:bnd_comp}) compares the behaviour of the two upper bounds, the one from i.i.d.~smoothing $u_\mathcal{I}$ and the one from generalized Gaussian smoothing $u_\mathcal{G}$, with respect to the Gaussian certificate $r_p$ obtained in equation~(\ref{eq:gaussian_bound}).
Assuming the binary classification case, for which $p_2(x) = 1 - p_1(x)$, we plot the ratios
\begin{align*} 
\frac{u_\mathcal{I}}{r_p} &= \frac{1}{\phi^{-1}(p_1(x)) \sqrt{2 (1-p_1(x))}},\\
\frac{u_\mathcal{G}}{r_p} &= \frac{4 \sqrt{\log \frac{1}{1-p_1(x)}}}{\phi^{-1}(p_1(x))}
\end{align*}
which only depend on $p_1(x)$ and show that the generalized Gaussian bound is much tighter than the i.i.d.~bound when $p_1(x)$ is close to one.
\begin{figure}[t]
\centering
\includegraphics[scale=0.2]{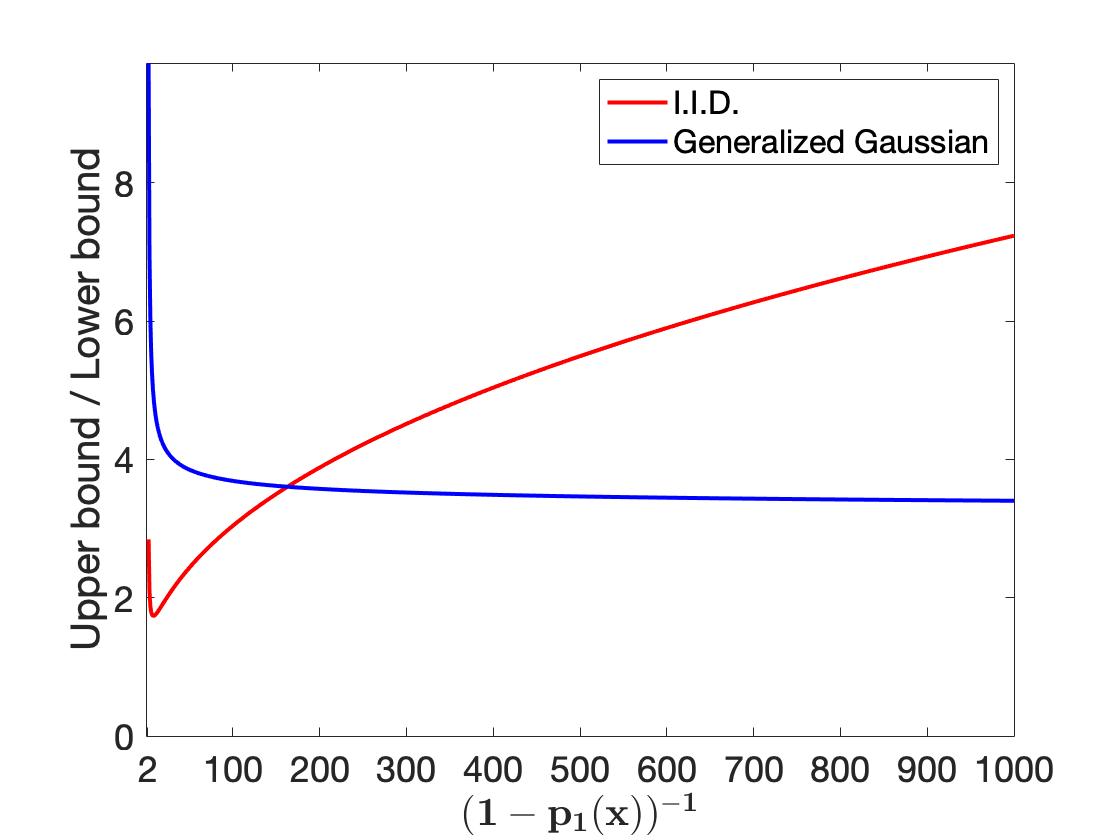}\vspace{-3mm}
\caption{Comparison of the upper bounds from i.i.d.\! smoothing~(\ref{eq:iid_bnd}) and generalized Gaussian smoothing~(\ref{eq:gen_gauss_bnd}) w.r.t.\! the lower bound obtained from Gaussian smoothing~(\ref{eq:gaussian_bound}). The x-axis represents $\frac{1}{1-p_1(x)}$ for $\frac{1}{2} \leq p_1(x) \leq 1$ and the y-axis represents the ratio of each upper bound to the Gaussian lower bound. At around $p_1(x) \approx 0.99$, the generalized Gaussian bound becomes tighter than the i.i.d.\! bound and gets within a constant factor of the Gaussian lower bound as $p_1(x)$ gets larger.}
\label{fig:bnd_comp}
\end{figure}

\section{Uniform Smoothing}
In this section, we analyse smoothing distributions that are uniform within a finite region around the input point $x$. We show stronger upper bounds for $r_p^*$ when smoothed uniformly over $\ell_1$ and $\ell_\infty$-norm balls. We first consider the $\ell_\infty$ smoothing distribution which is a limiting case for the generalized Gaussian distribution for $q = \infty$. We set $\mathcal{P}$ to be $\mathcal{U}([-b, +b]^d)$ which denotes a uniform distribution over the points in $[-b, +b]^d$.

\begin{theorem}
\label{thm:l_inf_uniform}
For distribution $\mathcal{U}([-b, +b]^d)$, the largest $\ell_p$-radius $r^*_p$ that can be certified for all classifiers, is bounded as
\[r^*_p < \frac{2b}{d^{1 - \frac{1}{p}}} = 2 \sqrt{3} \sigma / d^{1 - \frac{1}{p}}.\]
where $\sigma^2 = b^2/3$ is the variance in each dimension.
\end{theorem}

\begin{proof}
Assume $x$ is at origin and let $x'$ be a point with every coordinate equal to $\epsilon$. Let $V_1$ and $V_2$ denote the sets $[-b, +b]^d$ and $[-b + \epsilon, b + \epsilon]^d$. Consider a classifier $g$ that maps every point in $V_1 - V_2$ to class one and every point in $V_2 - V_1$ to class two. See figure~\ref{fig:linf_shift}. Let $\rho$ denote the probability with which the smoothing distribution for $\bar{g}(x)$ samples from $V_1 - V_2$, which is equal to the probability with which the smoothing distribution for $\bar{g}(x')$ samples from $V_2 - V_1$, or
\begin{align*}
\rho & = \frac{(2b)^d - (2b - \epsilon)^d}{(2b)^d}\\
     & = \left( 1 - \left( 1 - \frac{\epsilon}{2b} \right)^d \right).
\end{align*}

For $\bar{g}$ to classify $x'$ into class one, we must have:
\begin{align*}
    p_1(x') & > p_2(x')\\
    p_1(x) - \rho & > p_2(x) + \rho\\
    \rho & < \frac{p_1(x) - p_2(x)}{2}\\
    \left( 1 - \left( 1 - \frac{\epsilon}{2b} \right)^d \right) & < \frac{1}{2} \tag*{$p_1(x) - p_2(x) \leq 1$}\\
    \epsilon & < 2b \left( 1 - 2^{-1 / d} \right) < 2b/d \tag*{$\left( 1 - 2^{-1/d} \right) < 1/d$}
\end{align*}
Since $\norm{x'}_p = \epsilon d^{1/p}$, the optimal radius,
\[r^*_p < 2b/d^{1 - \frac{1}{p}} = 2 \sqrt{3} \sigma / d^{1 - \frac{1}{p}}\]
where $\sigma^2$ is the variance of $\mathcal{U}(-b, b)$.
\end{proof}

\begin{figure}[t]
\centering
\includegraphics[scale=0.5]{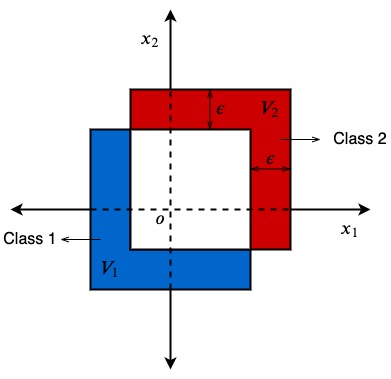}\vspace{-3mm}
\caption{2-D illustration of the $\ell_\infty$ smoothing case. The $\ell_\infty$ ball is shifted by $\epsilon$ along $x_1$ and $x_2$. The points in the blue region ($V_1 - V_2$) are mapped to class one and the points in the red region ($V_2 - V_1$) to class two.}
\label{fig:linf_shift}
\end{figure}

This shows that for $p>1$, $\sigma$ (or $b$) needs to grow with the number of dimensions $d$ to certify for a meaningfully large $p$-norm radius. For instance, $p = 2$ and $\infty$, require $\sigma$ to be $\Theta(\sqrt{d})$ and $\Theta(d)$ respectively. However, since inputs can be assumed to come from $[0, 1]^d$ (possibly after some scaling and shifting of images), smoothing over distributions with such large variance may significantly lower the performance of the smoothed classifier.

We now consider the uniform $\ell_1$ smoothing distribution (denoted by $\mathcal{L}_1(b)$) where points are sampled uniformly from an $\ell_1$-norm ball of radius $b$. Note that the noise in each dimension is no longer independent.
\begin{theorem}
\label{thm:l1_uniform}
For distribution $\mathcal{L}_1(b)$, the largest $\ell_p$-radius $r^*_p$ that can be certified for all classifiers, is bounded as
\[r^*_p < \frac{2b}{d}.\]
\end{theorem}

The following is a proof sketch of the above theorem.
Let $x$ be at the origin and $x'$ be the point $(\epsilon, 0, 0, \ldots, 0)$, that is, $\epsilon$ in the first coordinate and zero everywhere else. Similar to before, let $V_1$ and $V_2$ be the sets defined by the $\ell_1$ balls centered at $x$ and $x'$ respectively.
\begin{figure}[t]
\centering
\includegraphics[scale=0.6]{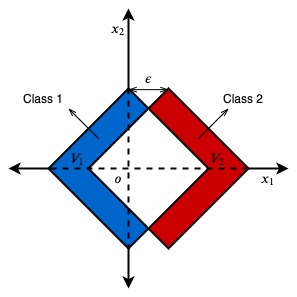}\vspace{-3mm}
\caption{2-D illustration of the $\ell_1$ smoothing case. The $\ell_1$ ball is shifted by $\epsilon$ along $x_1$. The points in the blue region ($V_1 - V_2$) are mapped to class one and the points in the red region ($V_2 - V_1$) to class two.}
\label{fig:l1_shift}
\end{figure}

\begin{lemma}
\label{lem:l1_intersect}
    The set $V_1 \cap V_2$ is a subset of an $\ell_1$ ball of radius $b - \frac{\epsilon}{2}$.
\end{lemma}
The proof is presented in the appendix.

As before, let $g$ be a classifier that maps every point in $V_1 - V_2$ to class one and every point in $V_2 - V_1$ to class two (figure~\ref{fig:l1_shift}). Let $\rho$ denote the probability with which the smoothing distribution for $\bar{g}(x)$ samples from $V_1 - V_2$, which is equal to the probability with which the smoothing distribution for $\bar{g}(x')$ samples from $V_2 - V_1$, or
\begin{align*}
\rho & \geq \frac{\frac{2^d}{d!} b^d - \frac{2^d}{d!}(b - \frac{\epsilon}{2})^d}{\frac{2^d}{d!} b^d}
     = \left( 1 - \left( 1 - \frac{\epsilon}{2b} \right)^d \right).
\end{align*}
We us the formula $2^d R^d/d!$ as the volume of a $d$-dimensional $\ell_1$ ball of radius $R$.
The rest of the analysis is same as that for the $\ell_\infty$ case and since $\norm{x'}_p = \epsilon$, we have,
\[r^*_p < \frac{2b}{d},\]
which proves Theorem~\ref{thm:l1_uniform}.

\section{Experiments}

\begin{figure}[t]
    \centering
    \includegraphics[scale=0.3]{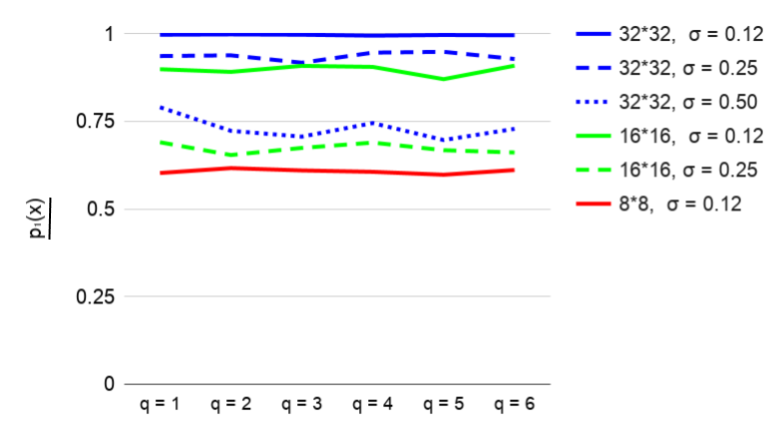}
    \caption{$\underline{p_1(x)}$ for CIFAR-10 images with median certified robustness for each classifier using Generalized Gaussian smoothing for different $q$. For a fixed standard deviation $\sigma$, the shape of the distribution, controlled by $q$, has almost no effect on the likelihood that the base classifier returns the correct class.}
    \label{fig:p_empirical}
\end{figure}

\begin{figure*}[t]
    \centering
    \includegraphics[width=\textwidth]{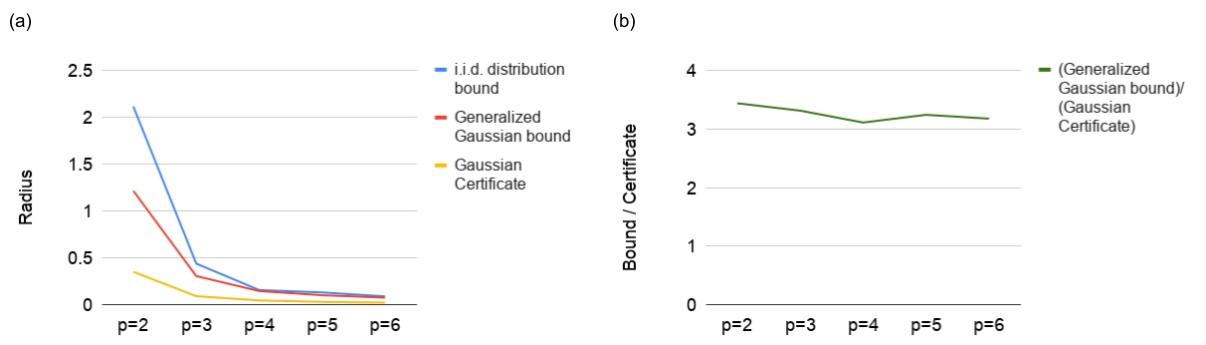}
    \caption{Upper bounds for certifying with Generalized Gaussian noise ($\sigma = .12$) on unaltered ($32\times 32$) CIFAR-10 images, with $q=p$, compared with certificates using Gaussian noise directly. At this noise level, $\underline{p_1(x)}$ is high enough for the Generalized Gaussian bound to be tighter than the i.i.d. distribution bound. Panel (a) shows the certificates and the bounds directly, while (b) shows the ratio between the tighter Generalized Gaussian bound and the certificate.}
    \label{fig:qpsigma123232}
\end{figure*}
\begin{figure*}[t]
    \centering
    \includegraphics[width=\textwidth]{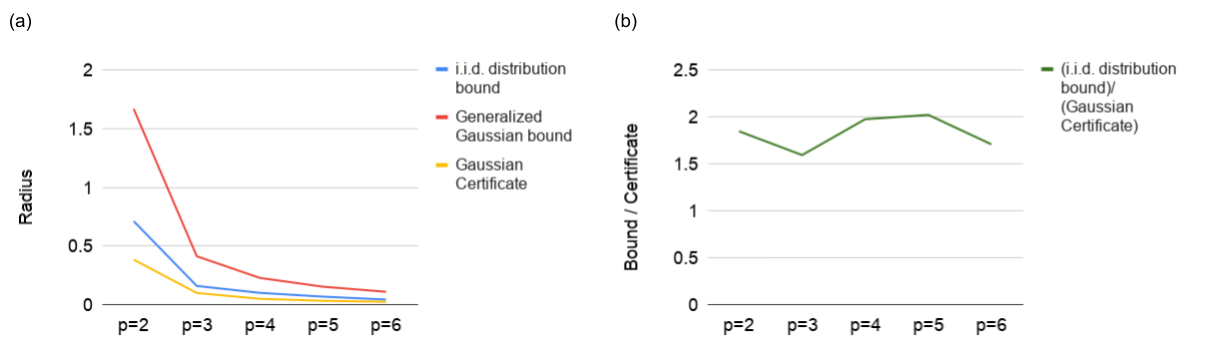}
    \caption{Repeating Figure~\ref{fig:qpsigma123232} for $\sigma = .25$. At this level of noise, $\underline{p_1(x)}$ is low enough so that the  i.i.d. distribution bound is tighter than the Generalized Gaussian bound (in contrast to the setup of Figure \ref{fig:qpsigma123232}).}
    \label{fig:qpsigma253232}
\end{figure*}

In order to understand how our results apply to smoothing in practice, we tested the smoothed classification algorithm proposed by \cite{cohen19}, using Generalized Gaussian noise in each dimension, rather than Gaussian noise. We specifically tested on CIFAR-10 ($32 \times 32$ pixels), as well as scaled-down versions of this dataset ($16 \times 16$ and $8 \times 8$ pixels), in order to study how our bounds behave as the dimension of the input changes. Although we do not have explicit certificates for these Generalized Gaussian distributions, we are able to compare the upper bounds derived in this work for any \textit{possible} certificates to the \textit{actual} certificates for Gaussian smoothing on the same images. Note that we re-trained the classifier on noisy images for each noise distribution and standard deviation $\sigma$.

Note also that our main results apply specifically to smoothing based certificates which are functions of only $p_1(x)$ and $p_2(x)$ (in theory, larger certificates could be derived if more information is available to the certification algorithm). In reporting the upper bounds on possible \textit{empirical} certificates, we provide the same inputs to the upper bound as we would provide to the certificate: namely, an empirical lower bound $\underline{p_1(x)}$ on $ p_1(x)$, estimated from samples, and an empirical upper bound  $\overline{p_2(x)}$ on $p_2(x)$. We are {\it not} making claims about the ``optimal possible'' empirical estimation procedures required to derive the largest possible certificates. We instead regard these bounds, $\underline{p_1(x)}$ and $\overline{p_2(x)}$, as \textit{inputs} to the empirical certificate: we are only claiming that, given estimates  $\underline{p_1(x)}$ and $\overline{p_2(x)}$, no certificate will exceed the computed bound. In practice, we use the estimation procedure proposed by \cite{cohen19}, which first selects a candidate top class label using a small number of samples, then uses a large number of samples ($100,000$ in our experiments) to compute $\underline{p_1(x)}$ based on a binomial distribution.  $\overline{p_2(x)}$ is then taken as $1-\underline{p_1(x)}$. Then, for the sake of our experiments, the only empirical input to our bound is the estimate of $\underline{p_1(x)}$.

\begin{figure*}[t]
\begin{minipage}{0.5\textwidth} \centering
    \includegraphics[scale=.35]{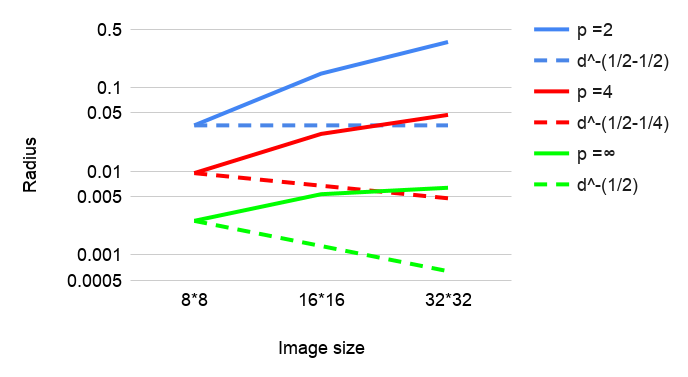}
    \caption{Certified Radius at different resolutions of CIFAR-10 using Gaussian noise $(\sigma=.12)$. The increase in accuracy of the base classifier on higher-resolution images overcomes the inverse scaling with $d$ in Eq.~\ref{eq:gaussian_bound}, achieving higher certified radii. Solid lines represent actual certificates and dashed lines represent how the certificates would scale if $\underline{p_1(x)}$ remained constant as resolution increased.}
    \label{fig:dim_scaling}
\end{minipage}%
~
\begin{minipage}{0.5\textwidth} \centering
    \includegraphics[scale=.35]{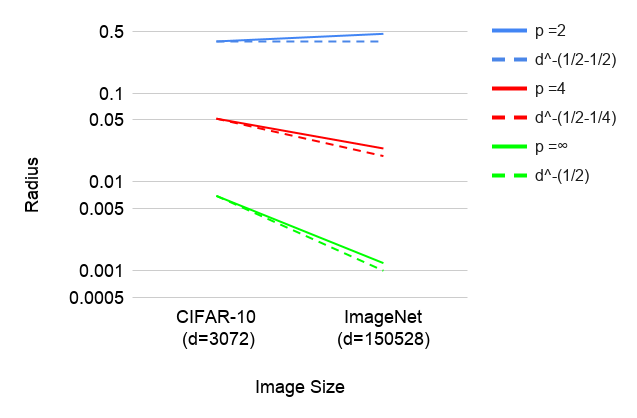}
    \caption{Certified Radius using Gaussian noise $(\sigma=.25)$, for datasets of different image resolutions. We see that for $p>2$, the certificates (solid lines) decrease with higher dimensionality almost as quickly as one would expect from the explicit dependence on $d$ in Equation \ref{eq:gaussian_bound} (dashed lines).}
    \label{fig:ImageNetvsCIFAR}
\end{minipage}
\end{figure*}

One interesting result is that the distribution of noise added in each dimension seems to be largely irrelevant to determining $\underline{p_1(x)}$ (Figure \ref{fig:p_empirical}). It is the variance of the noise added, \textit{not} the specific choice of noise distribution, that determines $\underline{p_1(x)}$. This paints an even bleaker picture for the possibility of smoothing for high $p$-norm robustness than our theoretical results alone can: Theorems \ref{thm:gen_iid} and \ref{thm:gen_gauss} still depend on $p_1(x)$ and $p_2(x)$ for the particular noise distribution used. This leaves open the possibility that certain choices of noise distributions could yield values of $p_1(x)$ large enough to counteract the scaling with $p$. However, empirically, we find that this is not the case: for a fixed $\sigma$,  $p_1(x)$ does not depend on the shape of the smoothing distribution. 

For example, one might attempt to use smoothing with $q = p$ in order to certify for the $\ell_p$ norm, so that the level sets of the smoothing distribution correspond to $\ell_p$ balls around $x$. This is the technique used for $\ell_1$ certification by \cite{LecuyerAG0J19}, and for $\ell_2$ certification by \cite{cohen19}. However, we find (Figures \ref{fig:qpsigma123232}, \ref{fig:qpsigma253232}) that, as anticipated by Figure \ref{fig:bnd_comp}, for $p>2$, this can only achieve at best a constant factor improvement in certified robustness compared to simply using Gaussian smoothing with the certificate from \cite{cohen19} and applying equivalence of norms (Equation \ref{eq:gaussian_bound}). Note that, as shown in Figure \ref{fig:p_empirical}, it was \textit{only} for the lowest level of noise tested ($\sigma = .12$) and the highest resolution images tested ($32 \times 32$) that  $\underline{p_1(x)}$ was sufficiently close to $1$ for the Generalized Gaussian bound to be tighter than the i.i.d. distribution bound (Figure \ref{fig:qpsigma123232}). For all other configurations (Figure \ref{fig:qpsigma253232}, other plots are given in supplementary materials) the i.i.d. bound is tighter.

In the case of Gaussian smoothing, \cite{cohen19} makes an argument that, as image resolution increases, the base classifier will become more tolerant to noise, because information will be redundantly encoded in the additional pixels. This should allow us to increase the magnitude of the smoothing variance $\sigma^2$ proportionally to $d$. It is because by average-pooling back down a large image to a low-resolution one, the variance in each pixel of the smaller image will decrease proportionally with $d$. Then, if it is possible to classify noisy images at the lower resolution with a certain accuracy $p_1(x)$, it should be possible to classify images at the higher resolution with higher levels of noise. This increase in the amount of noise that can be added to high resolution images (to obtain roughly the same accuracy to that of low resolution ones) will cancel out the decrease in the robustness radius due to the curse of dimensionality explained in this paper. It is because based on Equation \ref{eq:gaussian_bound}, if $\sigma$ is allowed to scale with $\sqrt{d}$ with $p_1(x)$ and $p_2(x)$ unchanged, then the certified radius should even remain constant with $d$ in the $\ell_\infty$ case.

For image datasets that are \textit{identical} except for a scaling factor, we observe a related phenomenon: for a fixed noise variance, $p_1(x)$ tends to increase with the resolution of the image (i.e., the dimensionality of the input), and therefore the certified radii tend to increase with $d$ in the $p=2$ case. In Figure \ref{fig:dim_scaling}, we show that, for $p > 2$, this increase is enough to counteract the \textit{inverse} scaling with $d$ in Equation \ref{eq:gaussian_bound}, at least in the case of low-resolution CIFAR-10 images. In other words, we still get larger certificates for larger-resolution images, simply because our base classifier becomes more accurate on noisy images as resolution increases. We emphasize that this is using the standard Gaussian noise: we have demonstrated that other i.i.d distributions will not give significantly better certificates.

The above setup, however, is an artificial scenario. In the real world, higher-resolution datasets are typically used for classification tasks which could \textit{not} be accomplished with high accuracy at a lower resolution. As shown in  Figure \ref{fig:ImageNetvsCIFAR}, if we compare, for a fixed $\sigma$, a real-world low dimensional classification task (CIFAR-10, $d=3072$) to a high dimensional classification task (ImageNet, $d=150528$), we see that the certified radius (and therefore $p_1(x)$), does \textit{not} substantially increase with higher resolution. Therefore, for higher $p$-norms, the certified radius decreases with dimension with a scaling nearly as extreme as the explicit $d^{(1/2 - 1/p)}$ factor in Equation \ref{eq:gaussian_bound}. Therefore, in practice, the curse of dimensionality can be observed as explained in this paper and it cannot be overcome using a novel choice of i.i.d. smoothing distribution.

\section{Conclusion}
In this work, we demonstrated some limitations of common smoothing distributions for $\ell_p$-norm bounded adversaries when $p > 2$. We partially answer the question, raised in \cite{cohen19}, whether smoothing techniques similar to Gaussian smoothing can be employed to achieve certifiable robustness guarantees for a general $\ell_p$-norm bounded adversary. Most i.i.d.\! smoothing distributions fail to yield good robustness guarantees in the high-dimensional regime against $\ell_p$-norm bounded attacks when $p > 2$. Their performance is no better than that of Gaussian smoothing up to a constant factor. While a constant factor improvement in performance could be critical in certain applications, the focus of this work is on the effect of dimensionality on certified robustness. We note that, in our analysis, we focus on i.i.d.\! and symmetric smoothing distributions. Our analysis highlights the importance of developing input-dependent smoothing techniques rather than the current smoothing methods based on i.i.d. distributions.

\section*{Software and Data}
The code for our experiments is available on GitHub at:

\url{https://github.com/alevine0/smoothingGenGaussian}


\section*{Acknowledgements}
We would like to thank anonymous reviewers for their valuable comments and suggestions. This project was supported in part by NSF CAREER AWARD 1942230, HR 00111990077, HR001119S0026 and Simons Fellowship on ``Foundations of Deep Learning.''




\bibliography{references}

\bibliographystyle{icml2020_simple}

\appendix
\section{Proof for lemma~\ref{lem:exp_bnd}}

\begin{proof}
Applying the series expansion of $e^{tZ}$, we get,
\begin{align*}
    E[e^{tZ}] &= \sum_{n=0}^{\infty} \frac{t^n E[Z^n]}{n!}\\
    E[Z^n] &= \frac{1}{C} \int_{-\infty}^{\infty} z^n e^{- \left( |z|/b \right)^q} dz\\
    &= \frac{1}{C} \int_{0}^{\infty} (1 + (-1)^n) z^n e^{-z^q/b^q} dz\\
    &= \begin{cases}
    0, &\text{$n$ is odd}\\
    \frac{2}{C} \int_{0}^{\infty} z^n e^{-z^q/b^q} dz, &\text{$n$ is even}\\
    \end{cases}
\end{align*}
When $n$ is even:
\begin{align*}
    E[Z^n] &= \frac{2}{C} \int_{0}^{\infty} z^n e^{-z^q/b^q} dz\\
    &= \frac{2 b^{n+1} \Gamma \left( \frac{n+1}{q} \right)}{C q}
\end{align*}
Substituting $C$,
\begin{align*}
    E[Z^n] &= \frac{b^n \Gamma \left( \frac{n+1}{q} \right)}{\Gamma(1/q)} \leq b^n \Gamma(n+1) \tag*{for $q \geq 1$}\\
    E[Z^n] &\leq b^n n!
\end{align*}
Therefore, keeping only the terms with even $n$ in the expansion of $E[e^{tZ}]$, we get:
\begin{align*}
    E[e^{tZ}] &\leq \sum_{m=0}^{\infty} (t^2 b^2)^m\\
    &= \sum_{m=0}^{\infty} \left( \frac{t^2 \sigma^2 \Gamma(1/q)}{\Gamma(3/q)} \right)^m \tag*{using $\sigma^2 = \frac{b^2 \Gamma(3/p)}{\Gamma(1/p)}$}\\
    &\leq \sum_{m=0}^{\infty} (c^2 t^2 \sigma^2)^m
\end{align*}
for some positive constant $c < 1.85$, because,
\begin{align*}
    \frac{\Gamma(1/q)}{\Gamma(3/q)} &= \frac{3 q \Gamma(1+1/q)}{q \Gamma(1+3/q)} \tag{using $\Gamma(z+1) = z \Gamma(z)$}\\
            &= \frac{3\Gamma(1+1/q)}{\Gamma(1+3/q)}\\
            &< 1.85^2 \tag{for $q \geq 1$, $\Gamma(1+1/q) \leq 1$ and $\Gamma(1+3/q)>0.88$}
\end{align*}
\end{proof}

\section{Proof for lemma~\ref{lem:l1_intersect}}

\begin{proof}
The points in $V_1$ satisfy the following $2^d$ constraints:
\begin{align*}
    x_1 + x_2 + \cdots + x_d & \leq b\\
  - x_1 + x_2 + \cdots + x_d & \leq b\\
    x_1 - x_2 + \cdots + x_d & \leq b\\
  - x_1 - x_2 + \cdots + x_d & \leq b\\
                             &\vdots\\
  - x_1 - x_2 - \cdots - x_d & \leq b\\
\end{align*}
Similarly, points in $V_2$ satisfy,
\begin{align*}
    (x_1 - \epsilon) + x_2 + \cdots + x_d & \leq b\\
  - (x_1 - \epsilon) + x_2 + \cdots + x_d & \leq b\\
    (x_1 - \epsilon) - x_2 + \cdots + x_d & \leq b\\
  - (x_1 - \epsilon) - x_2 + \cdots + x_d & \leq b\\
                             &\vdots\\
  - (x_1 - \epsilon) - x_2 - \cdots - x_d & \leq b\\
\end{align*}
Then, the points in $V_1 \cap V_2$ must satisfy the following set of constraints constructed by picking constraints that have a $+$ sign for $x_1$ in the first set of constraints and a $-$ sign for $x_1$ in the second set.
\begin{align*}
    x_1 + x_2 + \cdots + x_d & \leq b\\
  - (x_1 - \epsilon) + x_2 + \cdots + x_d & \leq b\\
    x_1 - x_2 + \cdots + x_d & \leq b\\
  - (x_1 - \epsilon) - x_2 + \cdots + x_d & \leq b\\
                             &\vdots\\
  - (x_1 - \epsilon) - x_2 - \cdots - x_d & \leq b\\
\end{align*}
They may be rewritten as,
\begin{align*}
    (x_1 - \epsilon / 2) + x_2 + \cdots + x_d & \leq b - \epsilon / 2\\
  - (x_1 - \epsilon / 2) + x_2 + \cdots + x_d & \leq b - \epsilon / 2\\
    (x_1 - \epsilon / 2) - x_2 + \cdots + x_d & \leq b - \epsilon / 2\\
  - (x_1 - \epsilon / 2) - x_2 + \cdots + x_d & \leq b - \epsilon / 2\\
                             &\vdots\\
  - (x_1 - \epsilon / 2) - x_2 - \cdots - x_d & \leq b - \epsilon / 2\\
\end{align*}
which define an $\ell_1$ ball of radius $b - \epsilon / 2$ centered at $(\epsilon / 2, 0, \ldots, 0)$, that is, $\epsilon / 2$ in the first coordinate and zero everywhere else.
\end{proof}

\section{Additional Plots of Certificate Upper Bounds}
See Figure \ref{fig:supfig}.
\begin{figure*}
    \centering
    \includegraphics[width=\textwidth]{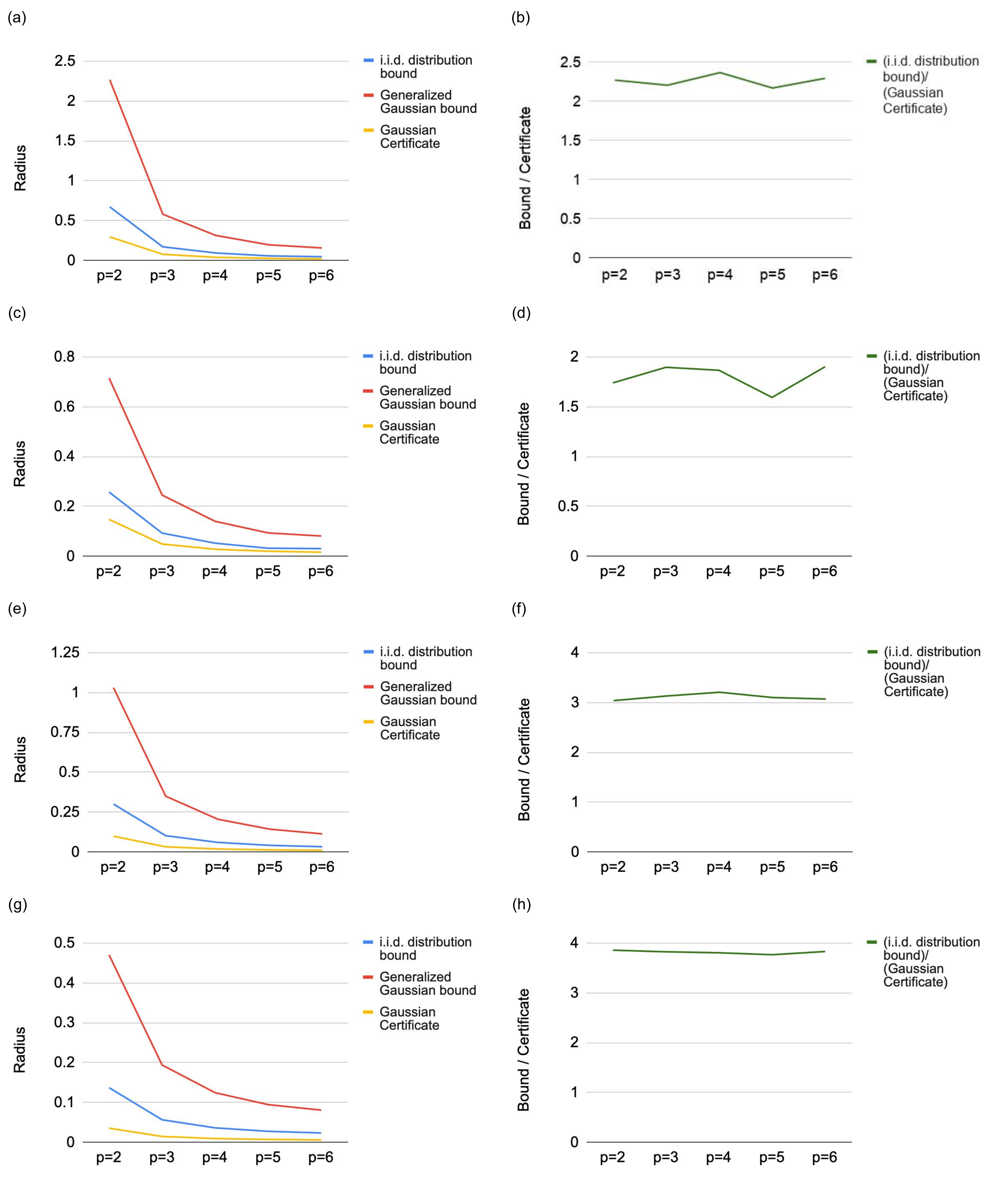}
    \caption{Upper bounds for certifying with Generalized Gaussian noise on CIFAR-10 images, with $q=p$, compared with certificates using Gaussian noise directly.Left panels show the certificates and the bounds directly, while right panels show the ratios between the i.i.d. distribution bounds (tighter in each case) and the certificates. Panels (a,b) use unaltered CIFAR-10 images with $\sigma = 0.5$ noise.  Panels (c,d) and (e,f) use CIFAR-10 images at $16 \times 16$ scale with $\sigma= 0.12$ and $\sigma=0.25$ respectively.  Panels (g,h) use CIFAR-10 images at $8 \times 8$ scale with $\sigma= 0.12$.  }
    \label{fig:supfig}
\end{figure*}
\section{Experimental Details}
Our experiments are adapted from the released code for $\ell_2$ smoothing from \cite{cohen19}. In particular, for each Generalized Gaussian distribution with varying parameter $q$ and standard deviation $\sigma$, we trained a ResNet-110 classifier on CIFAR-10 for 90 epochs, with the training under the same noise distribution as used for certification. All training and certification parameters are identical to those used in \cite{cohen19} unless otherwise specified. In particular, all certificates are reported to 99.9\% confidence, and we tested using a 500-image subset of the CIFAR-10 test set. For lower-resolution versions of CIFAR-10, we again trained separate models for each  resolution used, with  the resolution at training time matching  the resolution at test time. We first reduced the image resolutions before adding noise, then, once the noise was added, scaled the images back to the original $32  \times 32$ resolution (by repeating pixel values) before classifying with ResNet-110: this ensured that the number of parameters did not vary between classifiers.

We trained with $\sigma = 0.12, 0.25, 0.50, 1.00$ for resolutions $32 \times 32$, $16 \times 16$ and $8 \times 8$.  At higher levels of noise for each scale ($\sigma = 0.25$ for $8 \times 8$, $\sigma = 0.5$ for $8 \times 8$ and $16 \times 16$, $\sigma = 1.00$ on all scales) the resulting classifiers could not correctly certify the median image ($\underline{p_1(x)} < .5$), so we do not report any certificates. 

Values for ImageNet for the median certificate under Gaussian noise are adapted from the released certificate data from \cite{cohen19}.


\end{document}